\documentclass{new_tlp}
\usepackage{mathptmx}

\usepackage{latexsym}
\usepackage{amsmath}
\usepackage{amsfonts}
\usepackage{amssymb}
\usepackage[mathscr]{eucal}
\usepackage{graphicx}
\usepackage{url}

\def\squareforqed{\hbox{\rlap{$\sqcap$}$\sqcup$}}
\def\qed{\ifmmode\squareforqed\else{\unskip\nobreak\hfil
\penalty50\hskip1ex\null\nobreak\hfil\squareforqed
\parfillskip=0pt\finalhyphendemerits=0\endgraf\hskip1pt}\fi}

\newcommand{\rif}{~\ref}
\newcommand{\as}{``}
\def\ar{\leftarrow}
\newcommand{\no}{\hbox{\it not}\, }
\newcommand{\que}{\:?\!\!-} 
\newcommand{\be}{\begin{em}}
\newcommand{\ee}{\end{em}}

\newcommand{\tbs}{\hspace*{4mm}}
\newcommand{\tbm}{\hspace*{8mm}}

\newcommand{\yesA}{\mathit{yesA}}

\newcommand{\tabp}{\mbox{${\cal{T}}(\Pi)$}}
\newcommand{\ttabp}{{\cal{T}}(\Pi)}
\newcommand{\young}{\mathit{young}}
\newcommand{\old}{\mathit{old}}

\newenvironment{des}{
\begin{list}
{$\bullet$}
{\topsep = 2 mm
\labelwidth = 3 mm
\labelsep = 3 mm
 \parsep = 0.1 mm
\itemsep = \parskip
\leftmargin = 8 mm}
}{\end{list}}

\newtheorem{definition}{Definition}[section]
\newtheorem{proposition}{Proposition}[section]
\newtheorem{lemma}{Lemma}[section]
\newtheorem{theorem}{Theorem}[section]

\title[Query answering in resource-based answer set semantics]{Query Answering in Resource-Based\\ Answer Set Semantics~\thanks{This research
is partially supported by \textsc{YASMIN}~(RdB-UniPG2016/17) and FCRPG.2016.0105.021 projects.}}
\author[S.~Costantini and A.~Formisano]{Stefania Costantini\\
DISIM, Universit{\`a} di L'Aquila\\
{\email{stefania.costantini@univaq.it}}
\and Andrea Formisano\\
DMI, Universit{\`a} di Perugia --- GNCS-INdAM\\
{\email{formis@dmi.unipg.it}}
}
\begin{document}
\maketitle

\begin{abstract}
In recent work we defined resource-based answer set semantics, which is an extension
to answer set semantics stemming from the study of its relationship with linear logic.
In fact, the name of the new semantics comes from the fact that in the linear-logic 
formulation every literal (including negative ones) were considered as a resource.
In this paper, we propose a query-answering procedure reminiscent of Prolog for answer set
programs under this extended semantics as an extension of XSB-resolution
for logic programs with
negation.\footnote{A preliminary shorter version of this paper appeared in \cite{CostantiniF14}.}
We prove formal properties of the proposed procedure.

\noindent
Under consideration for acceptance in TPLP.
\end{abstract}

\section{Introduction}
Answer set programming (ASP) is nowadays a well-established and successful
programming paradigm based on answer set semantics \cite{GelLif88,MarTru99},  
with applications in many areas
(cf., e.g., \cite{Baral,TruszczynskiAppls07,Gelfond07}   
and the references therein).
Nevertheless, as noted in \cite{ebserGMS09,BonattiPS08}, few attempts to construct a goal-oriented proof
procedure exist, though there is a renewal of interest, as attested, e.g., by the recent work presented in \cite{MarpleG14}.
This is due to the very nature of the answer set semantics, where a program
may admit none or several answer sets, and where the semantics enjoys no locality, or, better, no
\emph{Relevance} in the sense of \cite{Dix95AeB}: no subset of the given program can in general
be identified, from where the decision of atom $A$ (intended as a goal, or query) belonging or not
to some answer set can be drawn.
An incremental construction of approximations of answer sets is proposed in \cite{ebserGMS09}
to provide a ground for local computations and top-down query answering.
A sound and complete proof procedure is also provided.
The approach of \cite{BonattiPS08} is in the spirit of \as traditional'' SLD-resolution \cite{lloyd93}, and
can be used with non-ground queries and with non-ground, possibly infinite, programs.
Soundness and completeness results are proven for large classes of programs.
Another way to address the query-answering problem is discussed in~\cite{LinYou02}. This work 
describes a canonical rewriting system 
that turns out to be sound and complete
under the \emph{partial stable model semantics}.
In principle, as the authors observe, the inference procedure could be completed to implement
query-answering w.r.t.\ stable model semantics by circumventing the lack of Relevance.
A substantially different approach to ASP computation is proposed in~\cite{gebsch06c} where
the authors define a tableau-based framework 
for ASP. The main aim consists in providing a formal framework for characterizing inference operations and
strategies in ASP-solvers. The approach is not based on query-oriented top-down evaluation, indeed, each
branch in a tableau potentially corresponds to a computation of an answer set.
However, one might foresee the possibility of exploiting such a tableau system to check
answer set existence subject to query satisfaction.

A relevant issue concerning goal-oriented answer-set-based computation is related to sequences of queries.
Assume that one would be able to pose a query $\que Q_1\,$ receiving an
answer \as yes'', to signify that $Q_1$ is entailed by some answer set of the given program $\Pi$.
Possibly, one might intend subsequent queries to be answered in the same \emph{context},
i.e.\ a subsequent query $\que Q_2\,$ might ask whether some of the answer sets entailing $Q_1$ also entails $Q_2$.
This might go on until the user explicitly \as resets'' the context.
Such an issue, though reasonable in practical applications, has been hardly addressed
up to now, due to the semantic difficulties that we have mentioned.
A viable approach to these problems takes inspiration from the research on
RASP (Resource-based ASP), which is a recent extension of ASP,
obtained by explicitly introducing the notion of \emph{resource} \cite{JLC-RASP}.    
A RASP and linear-logic modeling of default negation as understood under the answer set semantics has been introduced in \cite{JANCL-RASP-LL}.
This led to the definition of an extension
to the answer set semantics, called \emph{Resource-based Answer Set Semantics} (RAS).   
The name of the new semantics comes from the fact that in the linear-logic 
formulation every literal (including negative ones) is considered as a resource that is ``consumed'' (and hence it becomes no more available) once used in a proof. 
This extension finds an alternative equivalent definition in a variation of the auto-epistemic logic
characterization of answer set semantics discussed in \cite{MarekT93}.

We refer the reader to \cite{DBLP:journals/fuin/CostantiniF15} for a discussion of the new semantics from several
points of view, and to the Appendix for a summary of its formal definition.
Under resource-based answer set semantics there are no inconsistent programs, i.e., every program admits (resource-based) answer set.
Consider for instance the program \(\Pi_1 = \{\old \leftarrow \no \old\}\). Under the answer set 
semantics, this program is inconsistent (has no answer sets) because it consists of a unique odd cycle and no supported models exists.
If we extend the program to \(\Pi_2 = \{\old \leftarrow \no \old.\ \ \old \leftarrow \no \young.\}\)
then the resulting program has the answer set $\{old\}$: in fact, the first rule is overridden by the 
second rule which allows $\old$ to be derived.  
Under the resource-based answer set semantics the first rule is ignored in the first
place: in fact, $\Pi_1$ has a unique resource-based answer set which is the empty set. Intuitively, this results from
interpreting default negation $\no A$ as ``I assume that $A$ is false'' or, in autoepistemic terms \cite{MarTruAEL91,MarTruAELI91}
``I believe that I don't believe $A$''. So, since deriving $A$ accounts to denying the assumption of $\no A$,
such a derivation is disallowed as it would be contradictory. It is not considered to be inconsistent because
default negation is not negation in classical logic: in fact, the attempt of deriving $A$ from $\no A$ in classical
logic leads to an inconsistency, while contradicting one's own assumption is (in our view) simply meaningless,
so a rule such as the one in $\Pi_1$ is plainly ignored. Assume now to further enlarge the program,
by obtaining \(\Pi_3 = \{\old \leftarrow \no \old.\ \ \old \leftarrow \no \young.\ \ \young \leftarrow \old.\}\).
There are again no answer sets, because by combining the last two rules a contradiction
on $\young$ is determined, though indirectly. In resource-based answer set semantics there is still the answer set $\{\old\}$,
as the indirect contradiction is ignored: having assumed $\no \young$ makes $\young$ unprovable.

In standard ASP, a constraint such as $\leftarrow L_1,\ldots,L_h$ where the 
$L_i$s are literals is implemented by translating it into the rule
$p \leftarrow \no p, L_1,\ldots,L_h$ with $p$ fresh atom. This is because, in order to
make the contradiction on $p$ harmless, one of the $L_i$s must be false:
otherwise, no answer set exists. Under resource-based answer set semantics
such a transposition no longer works.
Thus, constraints related to a given
program are not seen as part of the program: rather, they
must be defined separately and associated to the program.
Since resource-based answer sets always exist, constraints will possibly
exclude (a-posteriori) some of them. Thus, constraints act as a filter on
resource-based answer sets, leaving those which are \emph{admissible}
with respect to given constraints.

In this paper we discuss a top-down
proof procedure for the new semantics. The proposed procedure, beyond query-answering,
also provides contextualization, via a form of tabling; i.e.,\ a table is associated 
with the given program, and initialized prior to posing queries. Such table
contains information useful for both the next and the subsequent queries.
Under this procedure, $\que A\,$ (where we us assume with no loss
of generality that $A$ is an atom), succeeds whenever there exists some
resource-based answer set $M$ where $A \in M$. Contextualization
implies that given a sequence of queries, for instance $\que A, \que B\,$, both queries succeed
if there exists some
resource-based answer set $M$ where $A \in M \wedge B \in M$:
this at the condition of evaluating $\que B\,$ on the program table as left by $\que A\,$
(analogously for longer sequences).
In case the table is reset, subsequent queries will be evaluated independently of previous ones.
Success of $\que A\,$ must then be validated with respect to constraints;
this issue is only introduced here, and will be treated in a future paper.

Differently from~\cite{ebserGMS09}, the proposed procedure 
does not require incremental answer set construction when answering a query
and is not based on preliminary program analysis as done in~\cite{MarpleG14}.
Rather, it exploits the fact that resource-based answer set semantics 
enjoys the property of Relevance \cite{Dix95AeB} (whereas answer set semantics does not). This guarantees that
the truth value of an atom can be established on the basis of the subprogram it depends upon,
and thus allows for top-down computation starting from a query.
For previous sample programs $\Pi_2$ and $\Pi_3$, query $\que \old$ succeeds, while $\que young$ fails.
W.r.t.~the top-down procedure proposed in \cite{BonattiPS08},
we do not aim at managing
function symbols (and thus programs with infinite grounding), so concerning this aspect our work is more limited.

As answer set semantics and resource-based answer set semantics extend the well-founded
semantics \cite{VGelderRS91}, we take as a starting point XSB-resolution \cite{SwiftW12,ChenW93},
an efficient, fully described and implemented procedure which is correct and
complete w.r.t.\ the well-founded semantics.
In particular, we define RAS-XSB-resolution and discuss
its properties; we prove correctness and completeness for every program (under the new semantics). 
We do not provide the full implementation details that we defer to a next step;
in fact, this would imply suitably extending and reworking all operative aspects related to XSB.
Thus, practical issues such as efficiency and optimization
are not dealt with in the present paper and are rather deferred to future work of actual
specification of an implementation.
The proposed procedure is intended 
as a proof-of-concept rather than as an implementation guideline.

RAS-XSB resolution can be used for answer set programming under
the software engineering discipline of dividing the program into a consistent \as base'' level
and a \as top'' level including constraints.
Therefore, even to readers not particularly interested in the new semantics,
the paper proposes a full top-down query-answering procedure for ASP,
though applicable under such (reasonable) limitation.

In summary, RAS-XSB-Resolution:
\begin{itemize}
\item
can be used for (credulous) top-down query-answering on logic programs under the resource-based answer set semantics and 
possibly under the answer set semantics, given the condition that constraints are defined separately
from the \as main'' program;
\item
it is meant for the so-called \as credulous reasoning'' in the sense that \ given, say, query $\que A\,$
(where $A$ is an atom), it determines whether there exists
any (resource-based) answer set $M$ such that $A \in M$;
\item
it provides \as contextual'' query-answering, i.e.\ it is possible to pose subsequent queries,
say $\que A_1,\ldots, \que A_n$ and, if they all succeed, this means that there exists some 
(resource-based) answer set $M$ such that $\{A_1,\ldots,A_n\} \subseteq M$; this extends to the 
case when only some of them succeed, where successful atoms are all in $M$ and unsuccessful ones are not;
\item
does not require either preliminary program analysis or incremental answer-set construction, and does
not impose any kind of limitation over the class of resource-based answer set programs which are considered
(for answer set programs, there is the above-mentioned limitation on constraints).
\end{itemize}

This paper is organized as follows. After a presentation of resource-based answer set semantics in Section\rif{rasintro}, we present the proposed query-answering procedure in Section\rif{procedure}, and conclude in Section\rif{conclusions}.
In the rest of the paper, we refer to the standard definitions concerning  propositional general logic programs and ASP \cite{lloyd93,Apt94,Gelfond07}.
If not differently specified, we will implicitly refer to the \emph{ground} version of a program $\Pi$.
We do not consider \as classical negation'',   
double negation $\no \no A$,
disjunctive programs, or the various useful programming constructs, such as aggregates, 
added over time to the basic ASP paradigm 
\cite{simonsEtAl,CosForLPNMR11,FaberPL11}.

\section{Background on Resource-based ASP}\label{rasintro}
The denomination \as resource-based'' answer set semantics (RAS) stems from the linear logic
formulation of ASP (proposed in \cite{JANCL-RASP-LL,DBLP:journals/fuin/CostantiniF15}),
which constituted the original inspiration for the new semantics.
In this perspective, the negation $\no A$ of some atom $A$ is
considered to be a \emph{resource} of unary amount, where:
\begin{itemize}
\item $\no A$ is \emph{consumed} whenever it is used in a proof, thus 
preventing $A$ to be proved, for retaining consistency;
\item $\no A$ becomes no longer available whenever $A$ is proved.
\end{itemize}

Consider for instance the following 
well-known sample 
answer set program consisting of a ternary odd cycle and
concerning someone who wonders where to spend 
her vacation: 

\smallskip\noindent{$
\begin{array}{l}
~~~~~~\mathit{beach} \ar \no \mathit{mountain}. ~~~~~
\mathit{mountain} \ar \no \mathit{travel}. ~~~~~
\mathit{travel} \ar \no \mathit{beach}. 
\end{array}$}

\smallskip\noindent
In ASP, such program is inconsistent. Under the new semantics, there are the following three resource-based answer sets:
~$\{\mathit{beach}\}$,~ $\{\mathit{mountain}\}$, and $\{\mathit{travel}\}$.
Take for instance the 
first one, $\{\mathit{beach}\}$.
In order to derive the conclusion $\mathit{beach}$ the first rule can be used; in doing so, the premise $\no \mathit{mountain}$ is consumed, thus disabling the possibility of proving $\mathit{mountain}$, which thus becomes false; $\mathit{travel}$ is false as well, since it depends from a false premise.
 
We refer the reader to \cite{DBLP:journals/fuin/CostantiniF15} for a detailed discussion about
logical foundations, motivations, properties, and complexity, and for examples of use. 
We provide therein characterizations of RAS in terms of linear logic, as a variation of the answer set semantics, and in terms of autoepistemic logic.
Here we just recall that, due to the ability to cope with odd cycles, under RAS it is always possible to
assign a truth value to all atoms:
every program in fact admits at least one (possibly empty) resource-based answer set.
A more significant example is the following (where, albeit in this paper we
	focus on the case of ground programs, for the sake of conciseness we make use of variables, as customary done to denote collections of ground literals/rules).
The program models a recommender agent, which provides a user with indication to where it is possible to spend the evening, and how the user should dress for such an occasion. The system is also able to take user preferences into account.

The resource-based answer set program which constitutes the core of the system is the following. There are two ternary cycles. The first one specifies that a person can be dressed either formally or normally or in an eccentric way. In addition, only old-fashioned persons dress formally, and only persons with a young mind dress in an eccentric way. Later on, it is stated by two even cycles that any person can be old-fashioned or young-minded, independently of the age that, by the second odd cycle, can be young, middle, or old. The two even cycles interact, so that only one option can be taken.
Then, it is stated that one is admitted to an elegant restaurant if (s)he is formally dressed, and to a disco if (s)he is dressed in an eccentric way. To spend the evening either in an elegant restaurant or in a disco one must be admitted. Going out in this context means either going to an elegant restaurant (for middle-aged or old people) or to the disco for young people, or sightseeing for anyone. 

\smallskip
\noindent{\small
$\tbm~formal\_dress(P)\ \ar \ person(P), \no normal\_dress(P), old\_fashioned(P).$\\
$\tbm~normal\_dress(P)\ \ar \ person(P), \no eccentric\_dress(P).$\\
$\tbm~eccentric\_dress(P)\ \ar \ person(P), \no formal\_dress(P), young\_mind(P).$\\
$\tbm~old(P)\ \ar \ person(P), \no middleaged(P).$\\
$\tbm~middleaged(P)\ \ar \ person(P), \no young(P).$\\
$\tbm~young(P)\ \ar \ person(P), \no old(P).$\\
$\tbm~old\_fashioned(P)\ \ar\ person(P), \no young\_mind(P), \no noof(P).$\\
$\tbm~noof(P)\ \ar \ person(P), \no old\_fashioned(P).$\\
$\tbm~young\_mind(P)\ \ar\ person(P), \no old\_fashioned(P), \no noym(P).$\\
$\tbm~noym(P)\ \ar\ person(P), \no young\_mind(P).$\\
$\tbm~admitted\_elegant\_restaurant(P)\ \ar\ person(P), formal\_dress(P).$\\
$\tbm~admitted\_disco(P)\ \ar\ person(P), eccentric\_dress(P).$\\
$\tbm~go\_disco(P)\ \ar person(P), young(P), admitted\_disco(P).$\\
$\tbm~go\_elegant\_restaurant(P)\ \ar person(P),admitted\_elegant\_restaurant(P).$\\
$\tbm~go\_elegant\_restaurant(P)\ \ar person(P), middleaged(P),admitted\_elegant\_restaurant(P).$\\
$\tbm~go\_sightseeing(P)\ \ar\ person(P).$\\
$\tbm~go\_out(P)\ \ar middleaged(P),go\_elegant\_restaurant(P).$\\
$\tbm~go\_out(P)\ \ar old(P),go\_elegant\_restaurant(P).$\\
$\tbm~go\_out(P)\ \ar young(P),go\_disco(P).$\\
$\tbm~go\_out(P)\ \ar go\_sightseeing(P).  $}
\smallskip

The above program, if considered as an answer set program, has a (unique) empty resource-based answer set, as there are no facts
(in particular there are no facts for the predicate $\mathit{person}$ to provide values for the placeholder $P$).

Now assume that the above program is incorporated into an interface system which interacts with a user, say George, who wants to go out and wishes to be made aware of his options.
The system may thus add the fact  $\mathit{person(george)}$ to the program. 
While, in ASP the program would become inconsistent, in RASP the system would, without any more information, advise George to go sightseeing.
This is, in fact, the only advice that can be extracted from the unique resource-based answer set of the resulting program. If the system might obtain or elicit George's age, the options would be many more, according to the hypotheses about him being old-fashioned or young-minded. Moreover, for each option (except sightseeing) the system would be able to extract 
the required dress code. George might want to express a preference, e.g., going to the disco. Then the system might add to the program the rule

\smallskip\noindent{
\(\tbm~\mathit{preference(P)}\ \ar \mathit{person(P),\ go\_disco(P)}.\)}

\smallskip\noindent
and state the constraint $\ar \no \mathit{preference(P)}$ that \as forces'' the preference to be satisfied,
thus making George aware of the hypotheses and conditions under which he might actually go to the disco.
Namely, they correspond to the unique resource-based answer set where George is young, young-minded and
dresses in an eccentric way.

However, in resource-based answer set semantics 
constraints cannot be modeled (as done in ASP) as \as syntactic sugar'', in terms of unary odd cycles
involving fresh atoms. Hence, they have to be modeled explicitly.
Without loss of generality, we assume, from now on, the following simplification concerning constraints.
Each constraint $\ar L_1,\ldots,L_k$, where
each $L_i$ is a literal, can be rephrased as simple constraint $\ar H$,
where $H$ is a fresh atom, plus rule $H \ar L_1,\ldots,L_k$ to be added to the given program $\Pi$.
So, $H$ occurs in the set $S_{\Pi}$ of all the atoms of $\Pi$.

\begin{definition}
	\label{admissible}
	Let $\Pi$ be a program and $\mathcal{C}=\{\mathcal{C}_1,\ldots,\mathcal{C}_k\}$ be a set
	of constraints, each $\mathcal{C}_i$ in the form $\ar H_i$. 
	\begin{itemize}
		\item
		A resource-based answer set $M$ for $\Pi$ is \emph{admissible} w.r.t.\ $\mathcal{C}$ 
		if for all $i \leq k$ where $H_i \not\in M$. 
		\item
		The program $\Pi$ is called \as admissible'' w.r.t.\ $\mathcal{C}$
		if it has an admissible answer w.r.t.\  $\mathcal{C}$.
	\end{itemize}
\end{definition}

It is useful for what follows to evaluate RAS with respect to general
properties of semantics of logic programs introduced in \cite{Dix95AeB},
that we recall below (see the mentioned article for the details).
A semantic $SEM$ for logic programs is intended as a function
which associates a logic program with a set of sets of atoms, which constitute the intended meaning.
\begin{definition}
Given any semantics $SEM$ and a ground program $\Pi$,
\emph{Relevance} states that for all literals $L$ it holds that $SEM(\Pi)(L) = SEM(\mathit{rel\_rul}(\Pi;L))(L)$.
\end{definition}
Relevance implies that the truth value of any literal under that
semantics in a given program, is determined solely by the subprogram consisting of the relevant rules.
The answer set semantics does not enjoy Relevance  \cite{Dix95AeB}. 
This is one reason for the lack of goal-oriented proof procedures. Instead, it is easy to see that resource-based answer set semantics enjoys Relevance.

Resource-based answer set semantics, like most semantics for logic programs with negation,
enjoys \emph{Reduction}, which simply assures that the
atoms not occurring in the heads of a program are always assigned truth value false.

Another important property is \emph{Modularity}, defined in~\cite{Dix95AeB} as follows
(where the reduct $\Pi^M$ of program $\Pi$ w.r.t.\ set of atoms $M$):

\begin{definition}
Given any semantics $SEM$, a ground program $\Pi$ let $\Pi = \Pi_1 \cup \Pi_2$
where for every atom $A$ occurring in $\Pi_2$, $\mathit{rel\_rul}(\Pi;A) \subseteq \Pi_2$.
We say that SEM enjoys \emph{Modularity}
if it holds that $\,SEM(\Pi) = SEM(\Pi_1^{SEM(\Pi_2)} \cup \Pi_2)$.
\end{definition}
If Modularity holds, then the semantics can be always computed
by splitting a program in its subprograms (w.r.t.\ relevant rules).
Intuitively, in the above definition, the semantics of $\Pi_2$, which is self-contained, is first computed. Then, the semantics 
of the whole program can be determined by reducing $\Pi_1$  w.r.t.\ $SEM(\Pi_2)$.
We can state (as a consequence of Relevance and of Proposition\rif{decomposition} in the Appendix) that resource-based answer set semantics enjoys Modularity.
\begin{proposition}
\label{modularity}
Given a ground program $\Pi$ let $\Pi = \Pi_1 \cup \Pi_2$,
where for every atom $A$ occurring in $\Pi_2$, $\mathit{rel\_rul}(\Pi;A) \subseteq \Pi_2$.
A set $M$ of atoms is a resource-based answer set of $\Pi$ iff there exists a resource-based answer set $S$ of $\Pi_2$
such that $M$ is a resource-based answer set of $\Pi_1^{S} \cup \Pi_2$.
\end{proposition}

Modularity 
also impacts on constraint checking, i.e.\ on the check of admissibility
of resource-based answer sets. Considering, in fact, a set of constraints $\{\mathcal{C}_1,\ldots,\mathcal{C}_n\}$, $n > 0$,
each $\mathcal{C}_i$ in the form $\ar{}H_i$, and letting for each
$i \leq n$ $\mathit{rel\_rul}(\Pi;H_i) \subseteq \Pi_2$, from Proposition\rif{modularity} it follows that,
if a resource-based answer set $X$ of $\Pi_2$ is admissible (in terms of Definition\rif{admissible}) w.r.t.\ $\{\mathcal{C}_1,\ldots,\mathcal{C}_n\}$,
then any resource-based answer set $M$ of $\Pi$ such that $X \subseteq M$ is also admissible w.r.t.\ $\{\mathcal{C}_1,\ldots,\mathcal{C}_n\}$. 
In particular, $\Pi_2$ can be identified in relation to a certain query:
\begin{definition}\label{relevantforquery}
Given a program $\Pi$, a constraint $\ar{}H$ associated to $\Pi$
is \emph{relevant for query} $\que A\,$ if
$\mathit{rel\_rul}(\Pi;A) \subseteq \mathit{rel\_rul}(\Pi;H)$.
\end{definition}

\section{A Top-down Proof Procedure for RAS}
\label{procedure}

As it is well-known, the answer set semantics extends the well-founded semantics (wfs) \cite{VGelderRS91}
that provides a unique three-valued model $\langle  W^{+}, W^{-}\rangle$,
where atoms in $W^{+}$ are \emph{true}, those in $W^{-}$ are \emph{false},
and all the others are \emph{undefined}. In fact, the answer set semantics assigns,
for consistent programs 
truth values
to the undefined atoms. However the program can be \emph{inconsistent} 
because of odd cyclic dependencies. 
The improvement of resource-based answer set semantics 
over the answer set semantics relies exactly on its ability to deal with odd cycles that
the answer set semantics interprets as inconsistencies.
So, in any reasonable potential query-answering device for ASP,
a query $\que A\,$ to an ASP program $\Pi$ may be reasonably expected
to succeed or fail if $A$ belongs to $W^{+}$ or $W^{-}$, respectively.
Such a procedure will then be characterized according to
how to provide an answer when $A$ is undefined under the wfs.

An additional problem with answer set semantics is that query $\que A\,$
might \emph{locally} succeed, but still, for the lack of Relevance, the overall program may not have answer sets.
In resource-based answer set semantics instead, 
every program has one or more resource-based answer set:
each of them taken singularly is then admissible or not w.r.t.\ the integrity constraints.
This allows one to perform constraint checking upon success of query $\que A$. 

We will now define the foundations of a top-down proof procedure for resource-based
answer set semantics, which we call RAS-XSB-resolution. The procedure has to deal
with atoms involved in negative circularities,
that must be assigned a truth value according to some resource-based
answer set. We build upon XSB-resolution, for which
an ample literature exists, from the seminal work in \cite{ChenW93} to the
most recent work in \cite{SwiftW12} where many useful references can also be found. 
For lack of space XSB-resolution is not described here. 
XSB in its basic version, XOLDTNF-resolution \cite{ChenW93}
is shortly described in the Appendix.
We take for granted basic notions concerning proof procedures for
logic programming, such as for instance backtracking. For
the relevant definitions we refer to \cite{lloyd93}.
Some notions are however required here for the understanding of what follows.
In particular, it is necessary to illustrate detection of cycles on negation.
\begin{definition}[XSB Negative Cycles Detection]
	\label{negcycdet}
	\begin{itemize}
		\item
		Each call to atom $A$ has an associated set $N$ of negative literals,
		called the \emph{negative context} for $A$, so the call takes the form $(A,N)$.
		\item
		Whenever a negative literal\, $\no B$\, is selected during the evaluation of some $A$, there are two
		possibilities: (i) $\no B \not\in N$: this will lead to the call $(B, N \cup \{\no B\})$;
		(ii) $\no B \in N$, then there is a possible negative loop, and $B$ is called a \emph{possibly looping negative literal}.
		\item
		For the initial call of any atom $A$, $N$ is set to empty.
	\end{itemize}
\end{definition}

In order to assume that a literal $\no B$ is a looping negative literal, that in XSB assumes truth value \emph{undefined},
the evaluation of $B$ must however be \emph{completed}, i.e.\ the search space must have been fully explored
without finding conditions for success or failure. 

Like in XSB, for each program $\Pi$ a table \tabp\ records useful
information about proofs. As a small extension w.r.t. XSB-Resolution, we 
record in \tabp\ not only successes, but also
failures.
XSB-resolution is, for Datalog programs, correct and complete w.r.t.\  the wfs.
Thus, it is useful to state the following definition.
\begin{definition}
	Given a program $\Pi$ and an atom $A$, we say that 
	\begin{itemize}
		\item
		$A$ \emph{definitely succeeds} iff it succeeds via
		XSB- (or, equivalently, XOLDTNF-) resolution, and thus $A$ is recorded in {\tabp} with truth value \emph{true}.
		For simplicity, we assume $A$ occurs in~{\tabp}.
		\item
		$A$ \emph{definitely fails} iff it fails via
		XSB- (or, equivalently, XOLDTNF-) resolution, and thus $A$ is recorded in {\tabp} with truth value \emph{false}.
		For simplicity, we assume {$\no A$} occurs in {\tabp}.
	\end{itemize}
\end{definition}

To represent the notion of negation as a resource, we initialize
the program table prior to posing queries and we manage the table
during a proof so as to state that: 
\begin{itemize}
	\item the negation of any atom
	which is not a fact is available unless this atom has been proved;
	\item the negation of an atom which has been proved becomes unavailable;
	\item the negation of an atom which cannot be proved is always available.
\end{itemize}

\begin{definition}
[Table Initialization in RAS-XSB-resolution]
\label{init-tabp}
Given a program $\Pi$ and  an associated table \tabp, \emph{Initialization} of \tabp\
is performed by inserting,
for each atom $A$ occurring as the
conclusion of some rule in $\Pi$, a fact $\yesA$ (where $\yesA$ is a fresh atom).
\end{definition}

The meaning of $\yesA$ is that the negation $\no A$ of $A$ has not been proved.
If $\yesA$ is present in the table, then $A$ can possibly succeed.
Success of $A$ \as absorbs'' $\yesA$ and prevents $\no A$ from success.
Failure of $A$ or success of $\no A$ \as absorbs'' $\yesA$ as well, but $\no A$ is asserted.
\tabp\ will in fact evolve
during a proof into subsequent states, as specified below.
\begin{definition}
	[Table Update in RAS-XSB-resolution]
	\label{upd-tabp}
Given a program $\Pi$ and  an associated table \tabp, 
referring to the definition of RAS-XSB-resolution (cf. Definition\rif{succfail} below), the  table update is performed as follows.
\begin{itemize}
\item
Upon success of subgoal $A$,  $\yesA$
is removed from \tabp
\ and $A$ is added to \tabp.
\item
Upon failure of subgoal $A$, $\yesA$ is removed from \tabp
\ and $\no A$ is added to \tabp.
\item
Upon success of subgoal $\no A$,
$\yesA$ is removed from \tabp\ and
$\no A$ is added to \tabp.
However:
\begin{itemize}
\item[-] if $\no A$ succeeds by case 3.b, then such modification is permanent;
\item[-]  if $\no A$ succeeds either by case 3.c or by case 3.d, then in case of failure of the parent subgoal the modification is retracted, i.e.\ 
$\yesA$ is restored in \tabp\ and $\no A$ is removed from \tabp.
\end{itemize}
\end{itemize}
\end{definition}

We refer the reader to the examples provided below for a clarification of the
table-update mechanism.
In the following, without loss of generality we can assume that a query is of the
form $\que A$, where $A$ is an atom.
Success or failure of this query is established as follows. 
Like in XSB-resolution, we assume that the call to query $A$ implicitly corresponds to 
the call $(A,N)$ where $N$ is the negative context of $A$, which is initialized to $\emptyset$
and treated as stated in Definition\rif{negcycdet}.
\begin{definition}
[Success and failure in RAS-XSB-resolution]
\label{succfail}
Given a program $\Pi$ and its associated table \tabp, notions of success and failure
and of modifications to \tabp\ are extended as follows with respect to XSB-resolution.
\begin{itemize}
\item[(1)]
Atom $A$ succeeds iff $\yesA$ is present in \tabp, and one of the following conditions holds.
\begin{itemize}
\item[(a)]
$A$ definitely succeeds (which includes the case where $A$ is present in \tabp).
\item[(b)]
There exists in $\Pi$ either fact $A$ or a rule of the form
$A \ar L_1,\ldots,L_n$, $n > 0$, such that neither $A$ nor $\no A$ occur in the body and every literal  $L_i$, $i \leq n$, succeeds.
\end{itemize}
\item[(2)]
Atom $A$ fails iff one of the following conditions holds.
\begin{itemize}
\item[(a)]
$\yesA$ is not present in \tabp.
\item[(b)]
$A$ definitely fails.
\item[(c)]
There is no rule of the form $A \ar L_1,\ldots,L_n$, $n > 0$, such that every literal $L_i$ succeeds.
\end{itemize}
\item[(3)]
Literal $\no A$ succeeds if one of the following is the case:
\begin{itemize}
\item[(a)]
$\no A$ is present in \tabp.
\item[(b)]
$A$ fails.
\item [(c)]
$\no A$ is \emph{allowed to succeed}.
\item[(d)]
$A$ is \emph{forced to failure}. 
\end{itemize}
\item[(4)]
Literal $\no A$ fails if $A$ succeeds.
\item[(5)]
$\no A$ is \emph{allowed to succeed}
whenever the call $(A,\emptyset)$ results,
whatever sequence of derivation steps is attempted, in the call $(A,N\cup\{\no A\})$. 
I.e., the derivation of $\no A$ incurs through layers of negation again into $\no A$.
\item[(6)]
$A$ is \emph{forced to failure} when the call $(A,\emptyset)$ always results in the call $(A, \{\no A\})$,
whatever sequence of derivation steps is attempted.
I.e., the derivation of $\no A$ incurs in $\no A$ directly.
\end{itemize}
\end{definition}

{F}rom the above extension of the notions of success and failure we obtain
RAS-XSB-resolution as an extended XSB-resolution. 
Actually, in the definition we exploit XSB (or, more precisely, XOLDTNF), as a \as plugin'' for definite
success and failure, and we add cases which manage subgoals with answer \emph{undefined} under XSB.
This is not exactly ideal from an implementation point of view. In future work, we intend to proceed to a much more
effective integration of XSB with the new aspects that we have introduced, and to consider efficiency
and optimization issues that are presently neglected.

Notice that the distinction between RAS-XSB-resolution and XSB-resolution
is determined by cases 3.c and 3.d of Definition\rif{succfail}, which manage literals 
involved in negative cycles. The notions of \emph{allowance to succeed} (case~5) and of \emph{forcing to failure} (case~6)
are crucial. 
Let us illustrate the various cases via simple examples:
\begin{itemize}
\item
Case 3.c deals with literals depending negatively upon themselves through other negations.
Such literals can be assumed as \emph{hypotheses}. Consider, for example, the program $a \ar \no b.\;~ b \ar \no a.~$
Query $\que a\:$ succeeds by assuming $\no b$, which is correct w.r.t.\ (resource-based) answer set $\{a\}$.
If, however, the program is $a \ar \no b,\no e.\;~ b \ar \no a.\;~ e.$ then, the same query $\que a\:$ fails upon definite failure
of $\no e$, so the hypothesis $\no b$ must be retracted. This is, in fact, stated in the specification
of table update (Definition\rif{upd-tabp}).
\item
Case 3.d deals with literals depending negatively upon themselves directly.
Such literals can be assumed as \emph{hypotheses}. Consider,  for example, the program $p \ar a.\;\ a \ar \no p.$.
Query $\que a\:$ succeeds because the attempt to prove $\no p$ comes across $\no p$  (through $a$),
and thus $p$ is forced to failure. This is correct w.r.t.\ resource-based answer set $\{a\}$.
Notice that for atoms involved in negative cycles the positive-cycle detection is relaxed, as some atom in
the cycle will either fail or been forced to failure.
If however the program is $p \ar a.\;\ a \ar \no p,\no q.$ then, the same query $\que a\:$ fails upon definite failure
of $\no q$, so the hypothesis $\no p$ must be retracted. This is in fact stated in the specification
of table update (Definition\rif{upd-tabp}).
\end{itemize}

We provide below a high-level definition of the overall proof procedure
(overlooking implementation details), which resembles plain SLD-resolution.

\begin{definition}[A naive RAS-XSB-resolution]
\label{ras-xsb}
Given a program $\Pi$, let assume as input the data structure \tabp\  used by the proof
procedure for tabling purposes, i.e.\ the table associated with the program.
Given a query $\que A$, the list of current subgoals
is initially set to ${\cal{L}}_1 = \{A\}$. 
If in the construction of a proof-tree for $\que A\,$
a literal $L_{i_j}$ is selected in the list
of current subgoals ${\cal{L}}_i$, we have that: if $L_{i_j}$ succeeds
then we take $L_{i_j}$ as proved and proceed to prove
$L_{i_{j+1}}$ after the related updates to the program table.
Otherwise, we have to backtrack to
the previous list ${\cal{L}}_{i-1}$ of subgoals. 
\end{definition}
Conditions for success and failure are those specified in Definition\rif{succfail}.
Success and failure determine the modifications to \tabp\
specified in Definition\rif{upd-tabp}.
Backtracking does not involve restoring previous contents of \tabp,
as subgoals which have been proved can be usefully employed as lemmas.
In fact, the table is updated only when the entire search space for a subgoal has
been explored.
The only exception concerns negative subgoals which correspond to literals involved
in cycles: in fact, they are to considered as hypotheses that could later be retracted.

For instance, consider the program
\\$\tbm~
q \ar \no a,c.~~~~~~~~~~
q \ar \no b.~~~~~~~~~~
a \ar \no b.~~~~~~~~~~
b \ar \no a.
$

\noindent 
and query $\que q$. Let us assume clauses are selected in the order. 
So, the first clause for $q$ is selected, and $\no a$ is initially allowed to succeed (though involved in a negative cycle with $\no b$).
However, upon failure of subgoal $c$ with consequent backtracking to the second rule
for $q$, lemma $\no A$ must be retracted from the table: this in fact enables $\no b$ to be allowed to succeed, so determining
success of the query.

\begin{definition}
Given a program $\Pi$ and its associated table \tabp,
a \emph{free query} is a query $\que A\,$ which is posed on $\Pi$ when the table
has just been initialized. A \emph{contextual query} is a query $\que B\,$
which is posed on $\Pi$ leaving the associated table in the state determined by former queries.
\end{definition}

Success of query $\que A\,$ means 
(as proved in Theorem\rif{correctcomplete} below) that there exist resource-based answer sets that contain $A$. The final content of \tabp\ 
specifies literals that hold in these sets (including $A$). Precisely, the state of \tabp\ characterizes a set ${\cal S}_{{\mbox{\tiny\tabp}}_A}$
resource-based answer sets of $\Pi$, such that for all $M \in {\cal S}_{{\mbox{\tiny\tabp}}_A}$, and for every atom $D$, $D \in  \tabp$ implies 
$D \in M$ 
and $\no D \in \tabp$ implies 
$D \not\in M$.
Backtracking on $\que A\,$ accounts to asking whether there are other different resource-based answer sets containing $A$, and implies making different assumptions about cycles by retracting literals which had been assumed to succeed. 
Instead, posing a subsequent query $\que B\,$ without resetting the contents of \tabp, which constitutes a \emph{context}, accounts to asking whether some of the 
answer sets in ${\cal S}_{{\mbox{\tiny\tabp}}_A}$ also contain $B$. 
Posing such a contextual query, the resulting table reduces previously-identified resource-based answer sets to a possibly smaller set 
${\cal S}_{{\mbox{\tiny\tabp}}_{A\cup B}}$
whose elements include both $A$ and $B$
(see Theorem\rif{contextcc} below).
Contextual queries and sequences of contextual queries are formally defined below.

\begin{definition}[Query sequence]
\label{q-seq}
Given a program $\Pi$ and $k > 1$ queries $\que A_1$, \ldots, $\que A_k$
performed one after the other, assume that
\tabp\ is initialized only before posing $\que A_1$. Then, $\que\,A_1\,$ is a {free query}
where each $\que A_i$, is a \emph{contextual query}, evaluated w.r.t.\ the previous ones.
\end{definition}

To show the application of RAS-XSB-resolution to single queries and to a query sequence, let us consider the 
sample following program $\Pi$, which includes virtually all cases of potential success
and failure.
The well-founded model of this program is $\langle \{e\}, \{d\}\rangle$ while
the resource-based answer sets are $M_1 = \{a,e,f,h,s\}$ and $M_2 = \{e,h,g,s\}$.
\\\centerline{\begin{tabular}{lclclcl}
		{
			$
			\begin{array}{l}
\phantom{\overline{\overline{|}}}			r_1.\ \ a\ar \no g.\\
\phantom{\underline{\underline{|}}} 			r_2.\ \ g\ar \no a.\\
			\end{array}
			$
		}
		&~~&
		{
			$
			\begin{array}{l}
			r_3.\ \ s\ar \no p.\\
			r_4.\ \ p\ar h.\\
			\end{array}
			$
		}
		&~~&
		{
			$
			\begin{array}{l}
			r_5.\ \ h\ar \no p.\\
			r_6.\ \ f\ar \no a,d.\\
			\end{array}
			$
		}
		&~~&
		{
			$
			\begin{array}{l}
			r_7.\ \ f\ar \no g,e.\\
			r_8.\ \ e.
			\end{array}
			$
		}
	\end{tabular}}

Initially, \tabp\  includes $\yesA$ for every atom occurring in some rule head:
$\ttabp=\{\mathit{yesa}$,$\mathit{yesb}$,$\mathit{yesc}$,$\mathit{yese}$,$\mathit{yesf}$,$\mathit{yesg}$,$\mathit{yesp}$,$\mathit{yesh}$,$\mathit{yess}\}$.
Below we illustrate some derivations.
We assume 
that applicable rules are considered from first ($r_1$) to last ($r_8$) as they are ordered in the program,
and literals in rule bodies from left to right.

Let us first illustrate the proof of query $\que f$.
Each additional layer of $\que$ indicates nested derivation of $A$
whenever literal $\no A$ is encountered.
In the comment, we refer to cases of RAS-XSB-resolution as specified in Definition\rif{succfail}.
Let us first consider query $\que f$.

\smallskip
{
	\noindent$
	\begin{array}{l@{\hspace{0.0ex}}l}
	\que f.\\
	\que \no a,d. \phantom{\que}&\tbs\%\ \mbox{via\ }r_6\\[0.2ex]
	\end{array}
	$

\noindent
Subgoal $\no a$ is treated as follows.

	\noindent$
	\begin{array}{l@{\hspace{0.0ex}}l}
	\que\que a. &\\
	\que\que \no g.\tbs\phantom{\que}&\%\ \mbox{via\ }r_1\\
	\que\que\que g.&\\
	\que\que\que \no a. \tbs&\%\ \mbox{via\ }r_2.\ \no a \mbox{\ succeeds by case 3.c,\ } \ttabp = \ttabp \cup \{\no a\} \setminus \{\mathit{yesa}\}\\[0.2ex]
	\end{array}
	$

\noindent
Subgoal $d$ gives now rise to the following derivation.

	\noindent$
	\begin{array}{l@{\hspace{0.0ex}}l}
	\que d. \tbs &\%\ \ d \mbox{\ fails by case 2.b, so the parent goal $f$ fails.} \\[0.5ex]
	\multicolumn{2}{l}{\!\!\!\!\mbox{Backtracking is however possible, as there exists a second rule for $f$.}}\\
	\que \no g,e. \tbs &\%\ \ \mbox{via\ }r_7\\
	\que\que g. &\\
	\que\que \no a. \tbs\phantom{\que} &\%\ \ \mbox{via\ }r_2\\
	\que\que\que a.\\
	\que\que\que \no g. \tbs &\%\ \ \mbox{via\ }r_1. \mbox{Thus,\ } \no g \mbox{\ succeeds by case 3.c.} \ttabp = \ttabp \cup \{\no g\} \setminus \{\mathit{yesg}\}
	\end{array}
	$
	
	\noindent$
	\begin{array}{l@{\hspace{0.0ex}}l}
	\multicolumn{2}{l}{\!\!\!\!\mbox{Now, the second subgoal $e$ remains to be completed:}}\\
	\que e. \tbs &\%\ \ e \mbox{\ succeeds by case 1.b,}\ \mbox{and the overall query $f$ succeeds by case 1.b.}\\
	&\phantom{\%\ \ }\ttabp = \ttabp \cup \{e, f\}\ \setminus\ \{\mathit{yese,yesf}\}
	\end{array}
	$
}

\medskip\noindent
Assuming now to go on to query the same context, i.e.\ without re-initializing
\tabp, {query $\que g\,$ quickly fails by case 2.a
since $\no g \in \ttabp$}.
Query $\que e\,$ succeeds immediately by case 1.a
as $e \in \ttabp$. We can see that the context we are within corresponds
to resource-based answer set~$M_1$. Notice that, if resetting the context,
 $\que g\,$ would instead succeed as by case 1.b as $\no a$ can be allowed
 to succeed by case 3.c.
Finally, a derivation for $\que s\,$ is obtained as follows:

\smallskip
{ 
\noindent
$\begin{array}{l@{\hspace{0.0ex}}l}
\que s.&\\
\que \no p. \tbs&\%\ \ \mbox{via\ }r_3\\
\que\que p.&\\
\que\que h.\tbs\phantom{\que}&\%\ \ \mbox{via\ }r_4
\\
\que\que \no p. \tbs\phantom{\que}&\%\ \ \mbox{via\ }r_5,\ 
{\no p \mbox{\ succeeds by case 3.d, and $p$ is forced to failure}
}\\
&\phantom{\%\ \ }\ttabp = \ttabp \cup \{\mathit{\no p}\} \setminus \{\mathit{yesp, yesh}\}.
\end{array}
$
}

\smallskip\noindent
Then, at the upper level,
$s$ and $h$ succeed by case 1.b, and
$\ttabp \cup \{s\} \setminus \{\mathit{yess}\}$.
Notice that forcing $p$ to failure determines $\no p$ to succeed, and consequently allows $h$ to succeed
(where $h$ is undefined under the wfs).
The derivation of $h$ 
involves the tricky case of a positive dependency through negation.

\subsection{Properties of RAS-XSB-resolution}\label{Sect_PropOfRAS-XSB}
Properties of resource-based answer set semantics are strictly related to properties
of RAS-XSB-resolution. In fact, thanks to
Relevance we have soundness and completeness, and
Modularity allows for contextual query and locality in constraint-checking.
Such properties are summarized in the following Theorems (whose proofs can be found in Appendix). 

\begin{theorem}
\label{correctcomplete}
\label{cc}
RAS-XSB-resolution is correct and complete w.r.t.\ resource-based answer set semantics,
in the sense that, given a program $\Pi$, a query $\que A\,$ succeeds under RAS-XSB-resolution
with an initialized \tabp\  iff there exists resource-based answer set $M$ for $\Pi$
where $A\in M$.
\end{theorem}
\begin{theorem}
	\label{contextcc}
RAS-XSB-resolution is contextually correct and complete w.r.t.\ resource Answer Set semantics,
in the sense that, given a program $\Pi$ and a query sequence $\que A_1$, \ldots, $\que A_k$, $k > 1$, where \(\{A_1, \ldots, A_k\} \subseteq S_{\Pi}\) (i.e.\ the $A_i$s are atoms occurring in $\Pi$),
we have that, for $\{B_1, \ldots, B_r\}\subseteq \{A_1, \ldots, A_k\}$ and $\{D_1, \ldots, D_s\} \subseteq \{A_1, \ldots, A_k\}$,
the queries $\que B_1$, \ldots, $\que B_r$ succeed while $\que D_1$, \ldots, $\que D_s$
fail under RAS-XSB-resolution,
iff there exists resource-based answer set $M$ for $\Pi$
where $\{B_1, \ldots, B_r\} \subseteq M$ and $\{D_1, \ldots, D_s\} \cap M = \emptyset$.
\end{theorem}
This result extends immediately to queries including negative literals such as $\no H$, $H \in S_{\Pi}$.
We say that a query sequence contextually succeeds if each of the involved queries succeeds in
the context (table) left by all former ones.

We defer a discussion of constraint checking to a future paper.
Notice only that, given an admissible program $\Pi$ 
and a constraint $\ar C$ (where $C$ is an atom), success of the query
$\que \no C\,$ in a certain context (given by $\tabp$) means that this constraint is fulfilled in the admissible resource-based answer sets $\Pi$ selected by that context.
If the context where $\que \no C\,$  is executed results from a query $\que A$, this implies by Theorem\rif{contextcc} that 
$\ar C$ is fulfilled at least one 
admissible resource-based answer set including $A$.
So, in admissible programs one should identify and check (a posteriori)
constraints that are \emph{relevant} to the query according to Definition\rif{relevantforquery}.

\section{Concluding Remarks}
\label{conclusions}

A relevant question about RAS-XSB-resolution is whether it might be applicable to non-ground queries and programs.
By resorting to standard unification,
non-ground queries on ground programs can be easily managed. 
In future work we intend however to extend the procedure to non-ground programs without requiring preliminary program grounding. This should be made possible
by the tabling mechanism, which stores ground positive and negative intermediate results,
and by Relevance and Modularity of resource-based answer set semantics.

An important issue is whether RAS-XSB-resolution might be extended to plain ASP.
Unfortunately, ASP programs may have a quite complicated structure: the effort of \cite{ebserGMS09} has been, in fact, that
of performing a layer-based computation upon some conditions.
Many answer set programs concerning real applications are however
already expressed with constraints at the top layer, as required by our approach. 

A comparison with existing proof procedures can be only partial, as these 
procedures cope with any answer set program, with its involved internal structure.
So, overall our procedure imposes less 'a priori' conditions and has a simple definition, but 
this is obtained by means of a strong preliminary assumption about constraints. However, as the expressive
power and complexity remain the same, our approach might constitute a 
way of simplifying implementation aspects without significant losses
in \as practical'' expressivity. 

We intend to investigate an integration of RAS-XSB-resolution 
with principles and techniques introduced in \cite{BonattiPS08}, so as to further enlarge its applicability
to what they call \emph{finitary programs}, which are a large class of non-ground programs
with function symbols. In fact, this approach allows programmers to make use of popular
recursive definitions which are common in Prolog, and makes ASP technology even more
competitive with respect to other state-of-the-art techniques.

In summary, we have proposed the theoretical foundations of a proof procedure
related to a reasonable extension of answer set programming. The procedure has been obtained by taking as a basis XSB-resolution and its
tabling features.
Future work includes a precise design of a RAS-XSB-resolution implementation. Our objective is to realize an efficient inference engine, that should then be checked and experimented
on (suitable versions of) well-established benchmarks (see, e.g., \cite{DBLP:journals/ai/CalimeriGMR16}).   
We intend in this sense to seek an integration with XSB, and with well-established ASP-related systems
(cf.~the discussion in \cite{GiunchigliaLM08}), already used for the implementation of the procedure proposed in \cite{BonattiPS08}.

\paragraph*{Acknowledgments}
The authors wish to thank the anonymous reviewers for their insightful comments.


\newpage
\appendix

\renewcommand\thesection{\Alph{section}}
This appendix contains background material concerning ASP (App.~\ref{APP-backasp}),
Resource-based ASP (App.~\ref{APP-resourceasp}),
and XSB-resolution (App.~\ref{APP-XSBnutshell}).
(All notions have been borrowed from the cited literature).
Appendix~\ref{ProofsFromSect_PropOfRAS-XSB} contains the proofs of the results
in  Section\rif{Sect_PropOfRAS-XSB}.


\section{Background on ASP}\label{APP-backasp}

We refer to the standard definitions concerning  propositional general logic programs,
as reported, for instance, in \cite{Apt94,lloyd93,GelLif88}.
We will sometimes re-elaborate definitions and terminology
(without substantial change), in a way which is functional to the
discussion.

In the answer set semantics (originally named \as stable model semantics''),
an answer set program $\Pi$ (or simply \as program'') is a finite
collection of \emph{rules} of the form
$H\leftarrow\; L_{1} , \ldots , L_n.$
where $H$ is an atom, $n\geqslant 0$ and each literal $L_i$ is either
an atom $A_i$ or its \emph{default negation} $\no A_i$.
The left-hand side and the right-hand side of rules are called
\emph{head} and \emph{body}, respectively.
A rule can be rephrased as $H\leftarrow\; A_{1} , \ldots , A_m, \no A_{m+1}, \ldots, \no A_n.$ where
$A_{1} , \ldots , A_m$ can be called \emph{positive body} and $\no A_{m+1}, \ldots, \no A_n$
can be called \emph{negative body}.\footnote{Observe that an answer set program can be seen as a Datalog program
	with negation ---cf.,~\cite{lloyd93,Apt94} for definitions about logic programming and Datalog.}
A rule with empty body ($n = 0$) is called a \emph{unit rule}, or \emph{fact}.
A rule with empty head, of the form $\leftarrow L_1,\ldots,L_n.$, is
a \emph{constraint},
and it states that the literals $L_1,\ldots,L_n$ cannot be simultaneously true.
A positive program is a logic program including no negative literals and no constraints.

For every atom $A$ occurring in a rule of program $\Pi$ either as positive literal $A$ or in a negative literal
$\no A$, we say that $A$ occurs in $\Pi$. Therefore, as $\Pi$ is by definition finite it is
possible to determine the set $S_{\Pi}$ composed of all the atoms occurring in $\Pi$.

In the rest of the paper, whenever it is clear from the context, by \as a (logic) program $\Pi$'' we mean
an answer set program (ASP program) $\Pi$.
As it is customary in the ASP literature, we will implicitly refer to the
\emph{ground} version of $\Pi$, which is obtained by replacing in all possible
ways the variables occurring in $\Pi$ with the constants occurring in $\Pi$ itself,
and is thus composed of ground atoms, i.e., atoms which contain no variables.
We do not consider \as classical negation'' (cf.,~\cite{GelLif91}), nor we consider
double negation $\no \no A$.
We do not refer (at the moment) to the various useful programming constructs defined
and added over time to the basic ASP paradigm.

A program may have several answer sets, or may have no answer set
(while in many semantics for logic programming 
a program admits exactly one \as model'', however defined).
Whenever a program has no answer sets, we will say that the program is \be inconsistent\ee.
Correspondingly, checking for consistency means checking for the existence of answer sets.

Consistency of answer set programs
is related, as it is well-known, to the occurrence of \emph{negative cycles}, (or negative \as loops'') i.e.\ cycles through negation,
and to their connections to other parts of the program (cf., e.g., \cite{Cos06}). 

To clarify this matter, some preliminary definitions are in order.

\begin{definition} [Dependency Graph]
	For a ground logic program $\Pi$, the dependency graph $G_\Pi$ is a finite directed graph whose vertices are
	the atoms occurring in $\Pi$ (both in positive and negative literals). There is a positive (resp. negative) edge from vertex $R$ to vertex $R'$
	iff there is a rule $\rho$ in $\Pi$ with $R$ as its head where $R'$ 
	occurs positively (resp. negatively) in its body, i.e.\ there is a positive edge if $R'$ occurs as a positive literal in the
	body of $\rho$, and a negative edge if  $R'$ occurs in a negative literal $\no R'$ in the
	body of $\rho$.
	We say that: 
	\begin{itemize}
		\item 
		$R$ depends on $R'$ if there is a path in $G_\Pi$ from $R$ to $R'$;
		\item
		$R$ depends positively on $R'$ if there is a path in $G_\Pi$ from $R$ to $R'$
		containing only positive edges;
		\item
		$R$ depends negatively on $R'$ if there is a path in $G_\Pi$ from $R$ to $R'$
		containing at least one negative edge.
		\item 
		there is an acyclic dependency of $R$ on $R'$ if there is an acyclic path in $G_\Pi$ from $R$ to $R'$;
		such a dependency is even if the path comprises an even number of edges, is odd otherwise.
	\end{itemize}
	In this context we assume that $R$ depends on itself only
	if there exist a non-empty path in $G_\Pi$ from $R$ to itself.
	(Note that empty paths are excluded, otherwise each $R$ would always depend -positively- upon itself by definition).
\end{definition}
By saying that atom $A$ depends (positively or negatively) upon atom $B$, we implicitly refer to the
above definition.

\begin{definition}[Cycles]\label{cycles}
	A cycle in program $\Pi$ corresponds to a circuit occurring in $G_\Pi$. 
	We say that: 
	\begin{itemize}
		\item
		a positive cycle is a cycle including only positive edges;
		\item
		a negative cycle is a cycle including at least one negative edge;
		\item
		given a negative cycle $C$, we say that $C$ is odd (or that $C$ is an odd cycle)
		if $C$ includes an odd number of negative edges; 
		\item
		given a positive cycle $C$, we say that $C$ is even (or that $C$ is an even cycle)
		if $C$ includes an even number of negative edges; 
	\end{itemize}
\end{definition}
When referring to positive/negative even/odd cycles we implicitly refer to the above definition.

Below is the formal specification of the answer set semantics, elaborated from~\cite{GelLif88}.
Preliminarily, we remind the reader that the
least Herbrand model of a positive logic program $\Pi$ can be computed 
by means of its immediate consequence operator $T_\Pi$, that can be defined as follows
(the original definition is due to Van Emden and Kowalski).
We then introduce the definition of reduct, the $\Gamma$ operator and finally the definition of answer set.
Given a positive program $\Pi$ and a set of atoms $I$, let
$$
T_\Pi(I) ~ = ~ \big\{A : \mbox{there exists a rule $A\leftarrow\: A_{1}, \ldots, A_m$ in $\Pi$ where $\{A_{1}, \ldots, A_m\}\subseteq I$}\big\}
$$

The $T_\Pi$ operator always has a unique least fixpoint, that for finite propositional programs is computable in a finite number of steps.

The following definition of (GL-)reduct is due to Gelfond and Lifschitz.
\begin{definition}\label{modulo}
	Let $I$ be a set of atoms and $\Pi$ a program.
	The reduct of $\Pi$ modulo $I$ is a new program, denoted as $\Pi^I$, obtained from $\Pi$ by:
	1. removing from $\Pi$ all rules which contain a negative literal\, $\no{}A$\, such that $A \in  I$; ~ and by
	2. removing all negative literals from the remaining rules.
\end{definition}

Notice that for each negative literal $\no A$ which is removed at step~2, it holds that
$A \not\in I$: otherwise, the rule where it occurs would have been removed at step~1.
We can see that $\Pi^I$ is a positive logic program.
Answer sets are defined as follows, via the GL-operator $\Gamma$. 
\begin{definition}[The GL-Operator $\Gamma$]
	Let $I$ be a set of atoms and $\Pi$ a program.
	We denote with $\Gamma_{\Pi}(I)$ the least Herbrand model of $\Pi^I$.
\end{definition}
\begin{definition}\label{asem}
	Let $I$ be a set of atoms and $\Pi$ a program.
	$I$ is an \emph{answer set} of $\Pi$ if and only if $\Gamma_{\Pi}(I)=I$.
\end{definition}

Answer sets form an anti-chain with respect to set inclusion.
The answer set semantics extends the well-founded semantics (wfs), formally introduced in \cite{VGelderRS91} and then further discussed and characterized (cf. \cite{Apt94} for a survey),
that provides a unique three-valued model.
The well-founded model $wfs_{\Pi} = \langle W^{+},W^{-} \rangle$ of program $\Pi$ is specified by making explicit the set of true and false atoms, 
all the other atoms implicitly assuming the truth value \as undefined''. 
Intuitively, according to the wfs:
\begin{itemize}
	\item
	The set $W^{+}$ is the set of atoms which can be derived top-down, say, like in Prolog, without incurring in cycles.
	\item
	The set $W^{-}$ is the set of atoms which cannot be derived either because they are not the head of any rule,
	or because every possible derivation incurs in a positive cycle, or because every possible derivation
	incurs in some atom which in turn cannot be derived.
	\item
	The undefined atoms are those atoms which cannot be derived because every possible derivation incurs in a negative cycle.
\end{itemize}

Some of the classical models of $\Pi$ (interpreted in the obvious way as a classical first-order theory, i.e.\
where the comma stands for conjunction and the symbol $\ar$ stands for implication) can be answer sets,
according to some conditions introduced in what follows.
\begin{definition}
	Given a non-empty set of atoms $I$ and a rule $\rho$ of the form $A \leftarrow\; A_1,\ldots,A_n,\no B_1, \ldots, \no B_m$, we say that $\rho$
	is \emph{supported} in $I$ iff 
	$\{A_1,\ldots,A_n\} \subseteq I$ and $\{B_1,\ldots,B_m\} \cap I = \emptyset$.
\end{definition}
\begin{definition}
	Given a program $\Pi$ and a non-empty set of atoms $I$, we say that $I$ is \emph{supported} w.r.t.~$\Pi$
	(or for short $\Pi$-supported)
	iff $\forall A \in I$, $A$ is the head of a rule $\rho$ in $\Pi$ which is
	supported in $I$.
\end{definition}

Answer sets of $\Pi$, if any exists, are supported minimal classical models of the program.
They however enjoy a stricter property, that we introduce below (cf., Proposition~\ref{asmcs}).
\begin{definition}
	\label{cs}
	Given a program $\Pi$ and a set of atoms $I$, an atom $A \in I$ is \emph{consistently supported} w.r.t.~$\Pi$ and $I$
	iff there exists a set $S$ of rules of $\Pi$ such that the following conditions hold
	(where we say that $A$ is consistently supported via $S$):
	\begin{enumerate}
		\item
		every rule in $S$ is supported in $I$;
		\item
		exactly one rule in $S$ has conclusion $A$;
		\item
		$A$ does not occur in the positive body of 
		any rule in $S$;
		\item
		every atom $B$ occurring in the positive body of some rule in $S$ is in turn
		consistently supported  w.r.t.\ $\Pi$ and $I$ via a set of rules $S' \subseteq S$.
	\end{enumerate}
\end{definition}

Note that $A$ cannot occur in the negative body of 
any rule in $S$ either, since all such rules are supported in $I$.
$S$ is called a \emph{consistent support set} for $A$ (w.r.t.~$\Pi$ and $I$). 
Moreover, by condition (ii), different support sets for $A$ may exist, each
one including a different rule with head~$A$.

\begin{definition}
	\label{mcs}
	Given a program $\Pi$ and a set of atoms $I$, we say that $I$ is a \emph{consistently supported} set of 
	atoms (w.r.t.\ $\Pi$) iff $\forall A \in I$, $A$ is consistently supported w.r.t.\ $\Pi$ and $I$. 
	We say that
	$I$ is a \emph{maximal consistently supported} set of atoms (MCS, for short)
	iff there does not exist $I' \supset I$ such that 
	$I'$ is consistently supported w.r.t.~$\Pi$. 
	We say, for short, that $I$ is an MCS for $\Pi$.
\end{definition}

Observe that an MCS can be empty only if it is unique, i.e, only if no non-empty consistently supported set of atoms exists.
In both the answer set and the well-founded semantics atoms involved/defined exclusively in positive cycles
are assigned truth value \emph{false}.
However, the answer set semantics tries to
assign a truth value to atoms involved in negative cycles,
which are \emph{undefined} under the well-founded semantics
(precisely, it succeeds in doing so if the given program $\Pi$ is consistent).
Therefore, for every answer set $M$, $W^+ \subseteq M$. It is easy to see that:

\begin{proposition}
	\label{wfmcs}
	Given the well-founded model $\langle W^{+},W^{-} \rangle$
	of program $\Pi$, $W^{+}$ is a consistently supported set of atoms.
\end{proposition}

Notice that $W^{+}$ is not in general an MCS, as the following proposition holds:
\begin{proposition}
	\label{asmcs}
	Any answer set $M$ of program $\Pi$ is an MCS for $\Pi$.
\end{proposition}

However, maximal consistently supported sets of atoms are not necessarily answer sets.

We introduce some useful properties of answer set semantics from \cite{Dix95AeB}.

\begin{definition}\label{def1}
	The sets of atoms a single atom $A$ depends upon, directly or indirectly,
	positively or negatively, is defined as~
	$ \mathit{dependencies\_of(A)} =  \{B: A \mbox{\ depends on\ } B\}$.
\end{definition} 

\begin{definition}\label{def2}
	Given a program $\Pi$ and an atom $A$,
	$\mathit{rel\_rul}(\Pi;A)$ is the set of relevant rules of $\Pi$ with respect to $A$, i.e.\ the set of rules
	that contain an atom $B \in (\{A\}\cup\mathit{dependencies\_of(A)})$ in their heads.
\end{definition}
The notions introduced by Definitions~\ref{def1} and~\ref{def2} for an atom $A$ can be plainly generalized to sets of atoms.
Notice that, given an atom (or a set of atoms) $X$, $\mathit{rel\_rul}(\Pi;X)$ is a subprogram of $\Pi$.

An ASP program can be seen as divided into components, some of them involving cyclic dependencies.
\begin{definition}
	\label{cyclicacyclic}
	An answer set program $\Pi$ is \emph{cyclic} if for every
	atom $A$ occurring in the head of some rule $\rho$ in $\Pi$, it holds that $A\in \mathit{dependencies\_of(A)}$.
	In particular, $\Pi$ is \emph{negatively} (resp., \emph{positively}) \emph{cyclic} if some (resp., none) of these dependencies is negative.
	A program $\Pi$ in which there is no head $A$ such that $A\in \mathit{dependencies\_of(A)}$ is called \emph{acyclic}.
\end{definition}

A cyclic program is not simply a program including some cycle: rather, it is a program where every atom is involved
in some cycle. 
It is easy to see the following.
\begin{itemize}
	\item
	An acyclic program has a unique (possibly empty)
	answer set, coinciding with the set $W^+$ of true atoms of its well-founded model.
	Acyclic programs coincide with \emph{stratified} programs in a well-known terminology \cite{Apt94}.
	We prefer to call them 'acyclic' as the notion of strata is irrelevant in the present context.
	\item
	A positively cyclic program has a unique empty
	answer set, coinciding with the set $W^+$ of true atoms of its well-founded model.
	\item
	Negatively cyclic programs have no answer sets and have an empty well-founded model, in the sense that all atoms
	occurring in such a program are undefined under the well-founded semantics. 
\end{itemize}

In the following, unless explicitly specified by a ``cyclic program'' (or program component) we intend a negatively cyclic program
(or program component, i.e.\ a subprogram of a larger program).
By Definition~\ref{cyclicacyclic}, there exist programs that are neither cyclic nor acyclic,
though involving cyclic and/or acyclic fragments as subprograms,
where such fragments can be either
independent of or related to each other.
\begin{definition}\label{self-contained}
	A subprogram $\Pi_s$ of a given program $\Pi$ is \emph{self-contained} (w.r.t.~$\Pi$)
	if the set $X$ of atoms occurring (either positively or negatively) in $\Pi_s$
	is such that $\mathit{rel\_rul}(\Pi;X) \subseteq \Pi_s$.
\end{definition}

Notice that a subprogram $\Pi_s = \Pi$ is self-contained by definition.

\begin{definition}
	\label{ontop}
	Given two subprograms $\Pi_{s_1},\Pi_{s_2}$ of a program $\Pi$, $\Pi_{s_2}$ is \emph{on top} of $\Pi_{s_1}$ if the set $X_2$ of atoms occurring in the head of some rule in $\Pi_{s_2}$ is such that $\mathit{rel\_rul}(\Pi;X_2) \subseteq \Pi_{s_2} \cup \Pi_{s_1}$, and the set $X_1$ of atoms occurring (either positively or negatively) only in the body of rules of $\Pi_{s_2}$ is such that $\mathit{rel\_rul}(\Pi;X_1) \subseteq \Pi_{s_1}$.\footnote{This notion was introduced in  \cite{Costantini1995,Lif94}.}
\end{definition}
Notice that, by Definition~\ref{ontop}, if $\Pi_{s_2}$ is \emph{on top} of $\Pi_{s_1}$, then $X_1$ is a \emph{splitting set} for $\Pi$ in the sense of \cite{Lif94}.

\begin{definition}
	\label{comp}
	A program obtained as the union of a set of 
	cyclic or acyclic programs, none of which is on top of another one, is called a \emph{jigsaw} program.
\end{definition}

Thus any program/component, either acyclic or cyclic or jigsaw, can possibly but not necessarily be
self-contained. An entire program is self-contained, but not necessarily jigsaw.
We introduce a useful terminology for jigsaw programs which are self-contained.

\begin{definition}
	\label{standalone}
	Let $\Pi$ be a program and $\Pi_s$ a jigsaw subprogram of $\Pi$.
	Then,  $\Pi_s$ is \emph{standalone} (w.r.t.~$\Pi$) if it is self-contained (w.r.t.~$\Pi$).
\end{definition}

In case we refer to a standalone program $\Pi_s$ without mentioning the including program~$\Pi$, we intend
$\Pi$ to be identifiable from the context.

The following property states that a program
can be divided into subprograms where a standalone one can be understood as the bottom
\emph{layer}, which is at the basis of a \as tower'' where each level is a jigsaw subprogram standing
on top of lower levels. 
\begin{proposition}
	\label{decomposition}
	A non-empty answer set program $\Pi$ can be seen as divided into a sequence of \emph{components}, or layers,
	$C_1, \ldots,C_n$, $n \geq 1$ where: $C_1$, which is called the \emph{bottom} of $\Pi$, is a standalone program;
	each component  $C_i$, for $i > 1$, is a jigsaw program which is on top of $C_{i-1} \cup \cdots \cup C_1$.
\end{proposition}

In fact, the bottom layer (that may coincide with the entire program) necessarily exists as the program is finite,
and so does any upper layer.
The advantage of such a decomposition is that, by the \emph{Splitting Theorem} introduced in \cite{Lif94},
the computation of answer sets of $\Pi$ can be divided into subsequent phases.

\begin{proposition}
	\label{incrementalas}
	Consider a non-empty ASP program $\Pi$, divided according to Proposition\rif{decomposition} into components $C_1, \ldots,C_n$, $n \geq 1$.
	An answer set $S$ of $\Pi$ (if any exists) can be computed incrementally as follows: 
	\begin{itemize}
		\item[step 0.] Set $i = 1$.
		\item[step 1.] Compute an answer set $S_i$ of component $C_i$ (for $i = 1$, this accounts to computing an answer set of the standalone bottom component).
		\item[step 2.] Simplify program $C_{i+1}$ by: (i) deleting all rules in which have $\no B$ in their body, for some $B \in S_i$;
		(ii) deleting (from the body of the remaining rules) every literal $\no F$ where $F$ does not occur in the head of rules of $C_{i+1}$ and
		$F \not\in S_i$, and every atom $E$ with $E \in S_i$.\footnote{Notice that, due to the simplification, $C_{i+1}$ becomes standalone.}
		\item[step 3.] If $i <n$ set $i = i+1$ and go to step 1, else set $S = S_1 \cup \cdots \cup S_n$.
	\end{itemize} 
\end{proposition}

All answer sets of $\Pi$ can be generated via backtracking (from any possible answer set of $C_1$,
combined with any possible answer set of simplified $C_2$, etc.).
If no (other) answer set of $\Pi$ exists, then at some stage step~1 will fail.
An incremental computation of answer sets has also been adopted in \cite{ebserGMS09}.

\section{Background on Resource-Based Answer Set Semantics}
\label{APP-resourceasp}

The following formulation of resource-based answer set semantics is obtained 
by introducing some modifications to the original definition of the answer set semantics.
Some preliminary elaboration is needed.
Following Proposition~\ref{decomposition},
a nonempty answer set program $\Pi$ (that below we call simply \as program'')
can be seen as divided into a sequence of \emph{components},
and, based upon such a decomposition, as stated in Proposition~\ref{incrementalas},
the answer sets of a program can be computed incrementally in a bottom-up fashion.
Resource-based answer sets can be computed in a similar way. Therefore, we start by defining the notion of resource-based
answer sets of standalone programs.

The semantic variation that we propose implies slight modifications in the definition of the $T_\Pi$ and the $\Gamma$ operator,
aimed at forbidding the derivation of atoms that necessarily depend upon their own negation. 
The modified reduct, in particular, keeps track of negative literals which
the \as traditional'' reduct would remove.

\begin{definition}
	\label{Modif-modulo} 
	\label{modifiedreduct}
	Let $I$
	be a set of atoms and let $\Pi$ be a program. The \emph{modified reduct} of $\Pi$ modulo $I$ is a new program,
	denoted as ${\hat{\Pi}}^I$, obtained from $\Pi$ by
	removing from $\Pi$ all rules which contain
	a negative premise $\no A$ such that $A \in  I$.
\end{definition}
For simplicity, let us consider each rule of a program as reordered by grouping its positive and its negative literals, as follows:
\[A\leftarrow\; A_{1} , \ldots , A_m,~ \no B_1,\ldots,\no B_n\]

\noindent
Moreover, let us define a \emph{guarded atom} to be any expression of the form $A || G$ where
$A$ is an atom and $G=\{\no C_1,\ldots,\no C_{\ell}\}$ is a possibly empty collection of $\ell\geq 0$ negative literals.
We say that $A$ is guarded by the $C_i$s, or that $G$ is a guard for $A$.

We define a modified $T_\Pi$ which derives only those facts that do not depend
(neither directly nor indirectly) on their own negation.
The modified $T_\Pi$ operates on sets of guarded atoms. For each inferred guarded atom $A || G$, the set $G$ records
the negative literals $A$ depends on.
\begin{definition}[Modified $T_\Pi$]\label{ModifiedT_P}\label{modifiedtp}
	Given a propositional program $\Pi$, let
	\\\centerline{$
		\begin{array}{lll}
		T_\Pi(I) & = & \Big\{A || G_1\cup\cdots\cup G_r\cup\{\no B_1,\ldots,\no B_n\} : \mbox{\ there exists a rule}\\
		& &\mbox{~\ $A\leftarrow\; A_{1} , \ldots , A_r, \no B_1,\ldots,\no B_n$~ in $\Pi$\  such that}\\
		& &\mbox{~\ $\{A_{1} || G_1 , \ldots , A_r || G_r\} \subseteq I$ \: and \:~$\no A \not\in\{\no B_1,\ldots,\no B_n\}\cup G_1\cup\cdots\cup G_r$}\ \Big\} .
		\end{array}
		$}
	
	\noindent
	For each $n\geq 0$,  let $T_\Pi^{n}$ be the set of guarded atoms defined as follows:
	\[
	\begin{array}{lll}
	T_\Pi^0 & = & \{A || \emptyset : \mbox{\ there exists unit rule $A\ar$ in $\Pi$}\}\\[0.2ex]
	T_\Pi^{n+1} & = & T_\Pi(T_\Pi^{n}) 
	\end{array}
	\]
	The  \emph{least contradiction-free Herbrand set} of $\Pi$ is the following set of atoms:
	\[{\hat T}_\Pi = \big\{A :\ A||G\ \in\  T_\Pi^i \mbox{ for some $i\geq 0$}\big\}.\]
\end{definition}
Notice that the least contradiction-free Herbrand set of a (modified reduct of a) program,
does not necessarily coincide with the full least Herbrand model
of the \as traditional'' reduct, 
as its construction excludes from the result those atoms that are guarded by their own negation.
We can finally define a modified version of the $\Gamma$ operator.
\begin{definition}[Operator ${\hat{\Gamma}}$]
	\label{Modif-gamma}
	Let $I$ be a set of atoms and $\Pi$ a program. Let $\hat{\Pi}^I$ be the modified reduct
	of $\Pi$ modulo $I$, and $J$
	be its least contradiction-free Herbrand set.
	We define $\hat{\Gamma}_{\Pi}(I)=J$.
\end{definition}

It is easy to see that given a program $\Pi$ and two sets $I_1$, $I_2$ of atoms, if $I_1 \subseteq I_2$ then
$\hat{\Gamma}_{\Pi}(I_1) \supseteq \hat{\Gamma}_{\Pi}(I_2)$.
Indeed, the larger $I_2$ leads to a potentially smaller
modified reduct, since it may causes the removal of more rules.

For technical reasons, we need to consider potentially supported sets of atoms.
\begin{definition}
	\label{psupp}
	Let $\Pi$ be a program, and let $I$ be a set of atoms. $I$ is \emph{$\Pi$-based} iff for any $A\in I$ there exists rule $\rho$ in $\Pi$ with head $A$.
\end{definition}

It can be shown (see, \cite{DBLP:journals/fuin/CostantiniF15}) that,
given  a standalone program $\Pi$ and a non-empty $\Pi$-based set $I$ of atoms,
and given $M$ $=$ $\hat{\Gamma}_{\Pi}(I)$, if $M \subseteq I$
then $M$ is a consistently supported set of atoms for $\Pi$.
Consequently, we have that $M$ is an MCS (cf., Definition~\ref{mcs}) for $\Pi$
iff there exists $I$ such that $M \subseteq I$, and there is no proper subset $I_1$ of $I$ such that $\hat{\Gamma}_{\Pi}(I_1) \subseteq I_1$.
We now define resource-based answer sets of a standalone program.

\begin{definition}
	\label{standalone-ras}
	Let $\Pi$ be a standalone program, and let $I$ be a $\Pi$-based set of atoms.
	$M$ $=$ $\hat{\Gamma}_{\Pi}(I)$ is a \emph{resource-based answer set} of $\Pi$
	iff $M$ is an MCS for $\Pi$.
\end{definition}

It is easy to see that any answer set of a standalone program $\Pi$ is a resource-based answer set of~$\Pi$ and,
if $\Pi$ is acyclic, the unique answer set of $\Pi$ is the unique resource-based answer set of $\Pi$.
These are  consequences of the fact that consistent ASP programs are
non-contradictory, and the modified $T_\Pi$, in absence of contradictions (i.e.\ in absence of atoms
necessarily depending upon their own negations), operates exactly like $T_\Pi$.
In case of acyclic programs, the unique answer set $I$ 
is also the unique MCS as the computation of the modified reduct does not cancel any rule,
and the modified $T_\Pi$ can thus draw the maximum set of conclusions, coinciding with $I$ itself.

Being an MCS, a resource-based answer set can be empty only if it is the unique resource-based answer set.

Below we provide the definition of resource-based answer sets of a generic program $\Pi$.

\begin{definition}
	{\label{ras}}
	Consider a non-empty ASP program $\Pi$, divided according to Proposition\rif{decomposition} into components $C_1, \ldots,C_n$, $n \geq 1$.
	A resource-based answer set $S$ of $\Pi$ is defined as $M_1 \cup \cdots \cup M_n$ where $M_1$ is a resource-based answer set of $C_1$,
	and each $M_i$, $1 < i \leq n$, is a resource-based answer set of standalone component $C_i'$, obtained by simplifying $C_i$
	w.r.t.~$M_1 \cup \cdots \cup M_{i-1}$, where the simplification consists in:
	(i) deleting all rules in $C_i$ which have $\no B$ in their body, $B \in M_1 \cup \cdots \cup M_{i-1}$;
	(ii) deleting (from the body of remaining rules) every literal
	$\no D$ where $D$ does not occur in the head of rules of $C_i$ and $D \not\in M_1 \cup \cdots \cup M_{i-1}$, and also every atom
	$D$ with $D \in M_1 \cup \cdots \cup M_{i-1}$.\footnote{Notice that, due to the simplification, $C_{i}'$ is standalone.}
\end{definition}

Definition\rif{ras} brings evident analogies to the procedure for answer set computation specified in Proposition\rif{incrementalas}.
This program decomposition is under some aspects reminiscent of the one adopted in \cite{ebserGMS09}.
However, in general, resource-based answer sets are not models in the classical sense: rather, they are $\Pi$-supported sets of atoms which are
the wider subsets of some classical model that fulfills non-contradictory support.
We can prove, in fact, the following result:

\begin{theorem}
	A set of atoms $I$ is a resource-based answer set of $\Pi$ iff it is an MCS for~$\Pi$. 
\end{theorem}

Resource-based answer sets still form (like answer sets) an anti-chain w.r.t.\ set inclusion,
and answer sets (if any) are among the resource-based answer sets. 
Clearly, resource-based answer sets semantics still extends the well-founded semantics.
Differently from answer sets,
a (possibly empty) resource-based answer set always exists. 

It can be observed that complexity remains the same as for ASP. In fact:

\begin{proposition}\label{NPcomplete}
	Given a program $\Pi$, the problem of deciding whether there exists a set of atoms $I$ which is
	a resource-based answer set of $\Pi$ is NP-complete.
\end{proposition}

\section{XSB-resolution in a Nutshell}\label{APP-XSBnutshell}
Below we briefly illustrate the basic notions of XSB-resolution.
An ample literature exists for XSB-resolution, from the seminal work in \cite{ChenW93} to the
most recent work in \cite{SwiftW12} where many useful references can also be found.
XSB resolution is fully implemented, and information and downloads can be find on
the XSB web site, \url{xsb.sourceforge.net/index.html}.

XSB-resolution adopts \emph{tabling}, that will
be useful for our new procedure.
Tabled logic programming was first formalized in the early 1980's, and several formalisms and systems have been
based both on tabled resolution and on magic sets, which can also be seen as a form of tabled logic programming
(c.f. \cite{SwiftW12} for references).
In the Datalog context, tabling simply means that
whenever atom $S$ is established to be true or false, it is recorded in a table. Thus, when subsequent calls are made to $S$, the evaluation
ensures that the answer to $S$ refers to the record rather than being re-derived using program rules.
Seen abstractly, the table represents the given state of a computation:
in this case, subgoals called and their answers so far derived.
One powerful feature of tabling is its ability to maintain other global elements of a computation
in the ``table'', such as information about whether one subgoal depends on another,
and whether the dependency is through negation.
By maintaining this global information, tabling is useful for evaluating logic programs under the well-founded semantics. 
Tabling allows Datalog programs with negation to terminate with polynomial data complexity under the well-founded semantics.

An abridged specification of the basic concepts underlying XSB-resolution
is provided  below for the reader's convenience. We refer the reader to the references for a proper understanding.
We provide explanations tailored 
to ground (answer set) programs, where a number of issues are much simpler
than the general case (non-ground programs and, particularly, programs with function symbols).
For definitions about procedural semantics of logic programs we again refer to \cite{lloyd93,Apt94},
and in particular we assume that the reader is to some extent acquainted with the SLD-resolution
(Linear resolution with Selection function for Definite programs)
and SLDNF-resolution (for logic programs with Negation-as-Failure) proof procedures, which
form the computational basis for Prolog systems. Briefly, a ground negative literal succeeds
under SLDNF-resolution if its positive counterpart finitely fails, and vice versa it fails if its positive
counterpart succeeds. SLDNF-resolution has the advantage of goal-oriented computation
and has provided an effective computational basis for logic programming, but it cannot be used as
inference procedure for programs including either positive or negative cycles.

XSB-resolution stems from SLS-resolution \cite{Przymusinski89,Ross92}, which is correct and complete w.r.t.
the well-founded semantics, via the ability to detect both positive cycles, which make involved
atoms false w.r.t.\  the wfs, and negative cycles, which make the involved atoms undefined.
Later, solutions with \as memoing'' (or \as tabling'') have been investigated, among which (for positive programs)
OLDT-resolution \cite{TamakiS86}, which maintains a table of calls and their
corresponding answers: thus, later occurrences of the same calls can be resolved using answers instead of
program rules. An effective variant of SLS with memoing and simple methods
for loop detection is XOLDTNF-resolution \cite{ChenW93}, which builds upon OLDT. SLG-resolution \cite{ChenW96} is a refinement of
XOLDTNF-resolution, and is actually the basis of implemented XSB-resolution.
In SLG, many software engineering aspects and implementation issues are taken into account.
In this context, as we still do not treat practical implementation issues 
it is sufficient to introduce basic concepts related to SLS and XOLDTNF-resolution.

As done before, let us consider each rule of a program as reordered by grouping its positive and its negative literals, as follows:
$A\leftarrow\; A_{1} , \ldots , A_m,~ \no B_1,\ldots,\no B_n$.
Moreover, let be given a \emph{goal} of the form $\leftarrow\; L_{1} , \ldots , L_k.$,
where the $L_i$s are literals, let us consider a \emph{positivistic computation rule}, which
is a computation rule that selects all positive literals before any negative ones.
These assumptions were originally required by SLS and have been dropped later, but they
are useful to simplify the illustration.

The basic building block of SLS-resolution is the \emph{SLP-tree}, which deals with goals
of the form $\leftarrow Q$, that form the root of the tree. For each positive subgoal which is encountered, 
its SLP sub-tree is built basically as done in SLD-resolution.
Leaves of the tree can be:
\begin{itemize}
	\item
	\emph{dead leaves}, i.e.\ nodes with no children because either there is no 
	program rule to apply to the selected atom $A$, or because $A$ was already selected
	in an ancestor node (situation which correspond to a positive cycle); in both case the node is \emph{failed};
	\item
	\emph{active leaves}, which are either empty (\emph{successful node}) or contain only negative subgoals.
\end{itemize}
More precisely, the \emph{Global tree} {\cal T} for goal $\leftarrow Q$ is built as follows.
\begin{itemize}
	\item 
	Its root node is the SLP-tree for the original goal.
	\item
	Internal tree nodes are SLP-trees for intermediate positive sub-goals. 
	\item
	\emph{Negation nodes} are created in correspondence of negative subgoals occurring in
	non-empty active leaves.
\end{itemize}

The management of negation node works as follows: the negation node corresponding to subgoal $\no A$
is developed into the SLP-tree for $A$, unless in case such a node already exists in the tree (negative cycles detection). 
Then: if some child of a negation node $J$ is a successful tree node, 
then $J$ is \emph{failed}; if every child of a negation node $J$ is either a failed node or
a dead leaf, then $J$ is \emph{successful}. 

Any node that can be proved 
successful or failed  is \emph{well-determined}, and any node which is not well-determined is \emph{undetermined}.
A \emph{successful branch} of {\cal T} is a branch that ends at a successful leaf and corresponds
to success of the original goal. A goal which leads via any branch to an undetermined node
is undetermined. Otherwise, the goal is failed.

It has been proved that a successful goal is composed of literals which are true w.r.t.\  the wfs,
a failed goal includes some literal which is false w.r.t.\  the wfs, and an undetermined goal
includes some literal which is undefined.

XOLDTNF-resolution augments SLS-resolution with tabling and with a simple direct way
for negative cycles detection. 
In the following, given a program $\Pi$, let \tabp\ be the data structure used by the proof
procedure for tabling purposes, i.e.\ the table associated with the program
(or simply \as program table''). The improvements of  XOLDTNF over SLS are mainly the following.
\begin{itemize}
	\item 
	The Global tree is split into several trees, one for each call, whose root is an atom~$A$.
	As soon as the call leads to a result, the \as answer'', i.e.\ the truth value of~$A$, is recorded in the table.
	Only \emph{true} or \emph{undefined} answers are explicitly recorded.
	Whenever $A$ should occur in a non-root node, it can be resolved only by the answer 
	that has been computed and recorded in \tabp\, or that can be computed later.
	This avoids positive loops. An atom whose associated tree has in the end no answer leaf has truth value \emph{false}
	because either no applicable program rule exists, or a positive cycle has been encountered.
	\item
	For detecting negative cycles the method introduced in Definition\rif{negcycdet} is adopted.
\end{itemize}

For Datalog programs,
XOLDTNF-resolution is, like SLS-resolution, correct and complete w.r.t.\  the wfs.
Consequently, so are SLG- and XSB-resolution. 

\section{Proofs from Section\rif{Sect_PropOfRAS-XSB}}\label{ProofsFromSect_PropOfRAS-XSB}
This section contains the proofs of Theorems~\ref{correctcomplete} and~\ref{contextcc} and
some preliminary results. 

\begin{lemma}\label{acycliccc}
Let $\Pi$ be an acyclic program. RAS-XSB-resolution is correct and complete w.r.t. such a program. 
\end{lemma}
\begin{proof} 
An acyclic program is stratified and thus admits 
a two-valued well-founded model (i.e.\ no atom is undefined)
where $W^{+}$ coincides with the unique (resource-based) answer set. XSB-resolution is correct and complete w.r.t.
such a program. Thus, any literal occurring in $\Pi$ either definitely succeeds by case 1 of RAS-XSB-resolution
or definitely fails by case 2.b of RAS-XSB resolution (cf., Definition\rif{succfail}). Since such cases just resort to plain XSB-resolution,
this concludes the proof.
\end{proof}

\begin{lemma}\label{cycliccc}
	Let $\Pi$ be a cyclic program. RAS-XSB-resolution is correct and complete w.r.t. such a program. 
\begin{proof}  
	Let $M$ be a resource-based answer set of $\Pi$.
We prove that, for every $A \in M$, query $\que A\,$ succeeds under RAS-XSB-resolution.
	$M$ (which is a maximal consistently supported set of atoms (MCS)) can be obtained by applying the
	modified immediate consequence operator) to some $\Pi$-based
set of atoms $I$.
	{F}rom the application of the modified $T_{\Pi^I}$ we can trace back a set of program rules from which $A$ can be 
	proved via RAS-XSB-resolution (cases 1 and 3 of Definition\rif{succfail}).
	Notice first that $T_{\Pi^I}^0 = \emptyset$ as a cyclic program includes no fact. (Recall that, by definition,  a program is cyclic if each of its heads depends
directly or indirectly on itself.)
	However, ${\Pi^I}$ necessarily contains some rule with body including negative literals
	only, thus leading to a nonempty $ {T^1_{\Pi^I}}$ and determining a final non-empty result of repeated application
	of ${T_{\Pi^I}}$. For some $i \geq 1$ there will be $A||G \in {T^i_{\Pi^I}}$ (for a guard $G$).
This means that there exists a rule $\rho$
	in $\Pi^I$ which is applicable, i.e.\ $A$ does not occur in its body,
	and $\no A$ does not occur in the guard. Let
	$B_1,\ldots,B_n,\no C_1,\ldots,\no C_m$, $n,m \geq 0$ be the body of $\rho$. 
	Since $M$ is an MCS $\rho$ will be supported in $M$, i.e.\ it will hold
that $B_i \in M$, $i \leq n$ and $C_j \not\in M$, $j\leq m$.

	Let us consider the $\no C_j$s. It cannot be $C_j \in I$, otherwise, by definition of the modified reduct, rule $\rho$ would have been canceled. Moreover, the $C_i$s are not derived by the modified $T_{\Pi^I}$ so allowing for the derivation of $A$. Being the program cyclic, one of the following must be the case for this to happen. 
	\begin{itemize}
	 \item
	 $C_j$ is not derived by the modified $T_{\Pi^I}$ (which differs from the standard one only concerning guarded atoms) because it depends positively upon itself and so it is false in every resource-based answer set and in the well-founded semantics. In this case $\no C_j$ succeeds by case 3.b of RAS-XSB-resolution: in fact $C_j$ fails by case 2.b since XSB-resolution is correct and complete w.r.t. the well-founded semantics.
	 \item
	 $C_j$ is not derived by the modified $T_{\Pi^I}$ because it depends negatively upon itself and at some point the derivation incurs in a guard including $\no C_j$. In this case, $\no C_j$ succeeds either by case 3.c or by case 3.d of RAS-XSB-resolution.
	 \end{itemize}
	 For each of the $B_i$s we can iterate the same reasoning as for $A$. As noted before, being the program cyclic there are no unit rules, but for $M$ to be nonempty there will exist some rule in $\Pi$ without positive conditions which is supported in $M$. Therefore, a RAS-XSB-derivation is always finite. This concludes this part of the proof.
	 
	 \smallskip
	  Let us now assume that $\que A$ succeeds by RAS-XSB-resolution. We prove that there exists resource-based answer set $M$ such that $A\in M$.
	  We have to recall that a resource-based answer set $M$ is obtained as $M$ $=$ $\hat{\Gamma}_{\Pi}(I)$ where $M \subseteq I$ for some set of atoms $I$,
	  and that $M$ is an $MCS$ for $\Pi$.
	  Let us refer to Definition\rif{succfail}. Since the program is cyclic, then 
	  $A$ succeeds via case 1.b, i.e.\ there exists a rule $\rho$ in $\Pi$ (where $A$ does not occur in the body), of the form
	  $A \ar B_1,\ldots,B_n,\no C_1,\ldots,\no C_m$ (for $n,m \geq 0$), where all the $B_i$s and all the $\no C_j$s succeed via RAS-XSB-resolution. We have to prove that there exists a resource-based answer set $M$, which is an MCS for $\Pi$, where this rule is supported, i.e.\ it holds that $B_i \in M$ for all $i \leq n$ and $C_j \not\in M$ for all $j\leq m$.
 {F}rom the definition of resource-based answer set, $M$ must be obtained from a set of atoms $I$, where we must assume to select an $I$ such that $A \in I$,  $\{B_1,\ldots,B_n\} \subseteq I$ and $\{C_1,\ldots,C_m\} \cap I = \emptyset$. So, the modified reduct will cancel all rules in $\Pi$ with $\no A$ in their body,
	  while keeping $\rho$. Thus, we have now to prove that $\rho$ allows the modified $T_{\Pi^I}$ to add $A$ to $M$. To this extent, we must consider both the negative and the positive conditions of $\rho$. 
	  Considering the negative conditions, for each the $\no C_j$s we can observe that, being $\Pi$ cyclic, one of the following must be the case.
	  \begin{itemize}
	  	\item $\no C_j$ succeeds via either case 3.c or 3.d. It can be one of the following.
	  	\begin{itemize}
	  		\item[-]
	  			All rules with head $C_j$ have been canceled by the modified reduct, and so the modified $T_{\Pi^I}$ cannot derive $C_j$.
	  			\item[-]
	  			There are rules with head $C_j$ which have not been canceled by the modified reduct, and might thus allow the modified $T_{\Pi^I}$ to derive $C_j$. Since however $\Pi$ is cyclic, the application of such a rule will be prevented by the occurrence of $\no C_j$ in the guard.
	  	\end{itemize}	   
	  		\item $\no C_j$ succeeds via case 3.b: in this case, being the program cyclic, $C_j$ depends in every possible way positively upon itself. Thus, $C_j$ cannot be derived by the modified $T_{\Pi^I}$ which, apart from guards, works similarly to the standard immediate consequence operator. 
	  	\end{itemize}

For each of the $B_i$s we can iterate the same reasoning as done for $A$, and this concludes the proof.
\end{proof}
\end{lemma}

\begin{lemma}\label{standalonecc}
	Let $\Pi$ be a standalone program. RAS-XSB-resolution is correct and complete w.r.t. such a program. 
\begin{proof}
The result follows from Lemma\rif{acycliccc} and Lemma\rif{cycliccc} as a standalone program is in general a jigsaw program including both cyclic and acyclic components.
\end{proof}
\end{lemma}

\begin{proof}[Proof of Theorem\rif{correctcomplete}]
As a premise, we remind the reader that, according to Definition\rif{ras},
for every resource-based answer set $M$ of $\Pi$ we have 
$M = M_1 \cup \ldots \cup M_n$, where $C_1 \cup \ldots \cup C_n$ are
the components of $\Pi$ and every $M_i$ is a resource-based answer set of the version of
$C_i$ obtained via the simplification specified in the same definition. For every $A \in M$,
there exists $i$, $1 \leq i \leq n$, such that $A \in M_i$.

Let $M$ be a resource-based answer set of $\Pi$. We prove that, for every $A \in M$,
query $\que A$ succeeds under RAS-XSB-resolution.
The proof will be by induction.

\medskip
\noindent
\emph{Induction base}.
Since $C_1$ is standalone, then by Lemma\rif{standalonecc} RAS-XSB-resolution is correct and complete w.r.t.\ $M_1$ and $C_1$.

\medskip
\noindent
\emph{Induction step}.
Assume that RAS-XSB-resolution is correct w.r.t. subprogram
 $C_1 \cup \ldots \cup C_i$, $i \leq n$, and its resource-based answer set $M_1 \cup \ldots \cup M_i$. We prove that this also holds for subprogram $C_1 \cup \ldots \cup C_{i+1}$ and its resource-based answer set $M_1 \cup \ldots \cup M_{i+1}$.
After the simplification specified in Definition\rif{ras}, which accounts to annotating in \tabp\ the results of the RAS-XSB derivations of the atoms in $M_{i+1}$, we have that $C_{i+1}$ becomes standalone, with resource-based answer set $M_{i+1}$.
Then, for $A \in M_{i+1}$ we can perform the same reasoning as for $A \in M_1$,
and this concludes the proof.
\end{proof}

\begin{proof}[Proof of Theorem\rif{contextcc}]
Given any query $\que A$,  the set of rules used in the derivation of $A$ constitutes a subprogram $\Pi_A$ of $\Pi$. Therefore, by correctness and completeness of RAS-XSB-resolution there exists some resource-based answer set $M_A$ of $\Pi_A$ such that, after the end of the derivation, we have $A \in \tabp \iff A \in M_A$  and $\no A \in \tabp \iff A \not\in M_A$ By Modularity of resource-based answer set semantics, there exists some resource-based answer set $M$ of $\Pi$ such that $M_A \subseteq M$ and therefore $A \in M$. 
So, let us assume that $\que A_1$ succeeds (if in fact it fails, then by correctness and completeness of RAS-XSB-resolution there exist no resource-based answer set of $\Pi$ including $A_1$, and by definition of RAS-XSB-resolution the table is left unchanged). For subsequent query $\que A_2$ one of the following is the case.
\begin{des}
	\item 
	The query succeeds, and the set of rules used in the derivation of $A_2$ has no intersection with the set of rules used in the derivation of $A_1$. Therefore, by Modularity of resource-based answer set semantics we have that $M_{A_1} \cap M_{A_2} = \emptyset$ and there exists resource-based answer set $M$ of $\Pi$ such that $(M_{A_1} \cup M_{A_2}) \subseteq M$.
	\item 
	The query succeeds, and the set of rules used in the derivation of $A_2$ has intersection with the set of rules used in the derivation of $A_1$. So, some literal in the proof will succeed by cases 1.a and 3.a of RAS-XSB-resolution, i.e, by table look-up. Therefore, by Modularity of resource-based answer set semantics we have that $M_{A_1} \cap M_{A_2} \neq \emptyset$ and there exists resource-based answer set $M$ of $\Pi$ such that $(M_{A_1} \cup M_{A_2}) \subseteq M$.
	\item 
	The query fails, and the set of rules attempted in the derivation of $A_2$ has no intersection with the set of rules used in the derivation of $A_1$. Therefore, we have that simply there not exists resource-based answer set $M$ such that $A_2 \in M$. 
	\item 
	The query fails, and the set of rules used in the derivation of $A_2$ has intersection with the set of rules used in the derivation of $A_1$. So, either some positive literal in the proof will fail by case 1.a of RAS-XSB-resolution or some negative literal in the proof will fail as its positive counterpart succeeds by case 1.a of RAS-XSB-resolution i.e, in both cases, by table look-up. So, success of $A_2$ is incompatible with the current state of the table, i.e.\ with success of $A_1$. Therefore, by Modularity of resource-based answer set semantics and by correctness and completeness of RAS-XSB-resolution we have that there not exists resource-based answer set $M$ such that $A_1 \in M$ and $A_2 \in M$ and $M_{A_1} \subseteq M$. 
\end{des}
The same reasoning can be iterated for subsequent queries, and this concludes the proof.
\end{proof}

\end{document}